\documentclass{article}
\usepackage{microtype}
\usepackage{graphicx}
\usepackage{subfigure}
\usepackage{booktabs}
\usepackage{xurl}
\usepackage{hyperref}

\usepackage[accepted]{icml2025}
\usepackage{amsmath,amssymb,amsthm}
\usepackage{algorithm,algorithmic}
\usepackage{listings}
\usepackage{textcomp}
\usepackage[T1,OT1]{fontenc}
\usepackage{xcolor}
\makeatletter
\renewcommand{\Notice@String}{Research preprint, October 2026. Not an ICML proceedings publication. Copyright 2026 by the authors.}
\makeatother
\hypersetup{pdfsubject={Research preprint on passive imitation and verified skill acquisition}}
\icmltitlerunning{Passive Imitation: Localized Repair and Verified Skill Acquisition}
\newtheorem{proposition}{Proposition}
\newtheorem{assumption}{Assumption}
\newcommand{\PI}{\textnormal{\textsc{PI}}}

\begin{document}
\twocolumn[
\icmltitle{Imitation Holds a Key to Evolution:\texorpdfstring{\\}{ }Localized Repair and Verified Skill Acquisition for Adaptive Agents}
\begin{icmlauthorlist}
\icmlauthor{Kai Hu}{bisu}
\icmlauthor{Hua Huang}{cqupt}
\end{icmlauthorlist}
\icmlaffiliation{bisu}{Beijing International Studies University}
\icmlaffiliation{cqupt}{Chongqing University of Posts and Telecommunications}
\icmlcorrespondingauthor{Kai Hu}{2025110321@stu.bisu.edu.cn}
\icmlkeywords{passive imitation, agent architecture, program repair, skill acquisition, continual adaptation}
\vskip 0.3in
]
\printAffiliationsAndNotice{Author email: Hua Huang \textless{}S251101018@std.cqupt.edu.cn\textgreater{}. }

\begin{abstract}
An agent can retain successful trajectories without learning which part of an experience should repair a new discrepancy.
We formulate passive imitation (\PI) as an architectural loop that localizes a violated contract, extracts a local goal, proposes an effect-matched skill, tests a small patch, and retains evidence after full-task verification.
We instantiate this loop as a deterministic symbolic controller and analyze its finite repair-query complexity under explicit identifiability and noninterference assumptions.
Across ten simulated tool worlds, \PI\ solves all 600 cold held-out composition tasks with three allowed small trials, compared with 6.17\% for a whole-trajectory control, while an equally expressive compositional repair control matches \PI.
On an unseen tool revision, \PI\ achieves 70.0\% success versus 24.0\% without small trials and 79.0\% with self-trial enumeration; an exploratory 24-task adaptation stream subsequently raises first-submission success to 100\%, supporting a limited correction-and-retention mechanism rather than a unique incentive advantage or general intelligence.
The CPU-only implementation, integrity checks, and reproducible results are available at \url{https://github.com/FriendOfYishan/PI}.
\end{abstract}

\section{Introduction}
An adaptive agent must decide what a remembered experience licenses it to change. A reporting workflow may select paid records, resolve duplicate revisions, and normalize monetary values. Copying it into a different request can repair one operation while damaging two others. A failed workflow may also contain useful components that support a smaller intervention.

We ask: \emph{can an agent convert a discrepancy into a locally verified skill that improves later decisions without updating a neural model?} This requires identifying the implicated computation, expressing its desired effect, testing a candidate transformation, and assessing whether the evidence supports reuse. Current-task success and reusable skill acquisition are separate events.

We call this organization \emph{passive imitation}. ``Passive'' refers to learning opportunities arising during an assigned task. Acquired skills should be used before a discrepancy recurs. Candidates can come from another execution, the agent's own history, or action enumeration. Imitation is an optional source of hypotheses; the reusable object is a transformation with a proposed effect and an applicability condition.

Existing agents combine reasoning, feedback, memory, and routines: ReAct interleaves reasoning and actions \citep{yao2023react}; Reflexion retains verbal feedback \citep{shinn2023reflexion}; ExpeL extracts experiential insights \citep{zhao2024expel}; and Agent Workflow Memory induces workflows \citep{wang2025awm}. Neither experience reuse nor localized correction is new in isolation. We make their proposed organization explicit and test its components in an inspectable setting.

Our small symbolic world removes language-model variance while exposing supplied assumptions: typed roles, an executable grammar, and structured assertions. Effect recognition and goal projection are programmed. We test the architecture with these interfaces available, without testing their discovery from unrestricted language.

Local patches outperform whole-trajectory reuse on withheld compositions, and verified records improve future first attempts. Small trials help when tool meanings change but are unnecessary for reliable familiar fragments. A strong compositional control ties \PI, while demonstration-free enumeration performs better on unseen semantics under a tight budget. These findings constrain the architecture's interpretation.

Our contributions are a controller separating retrieval, tentative repair, full verification, and reuse; a conditional query analysis explaining when promotion thresholds are redundant; and an exhaustively checked CPU study with matched budgets, an expressive control, ablations, transfer failures, and exploratory adaptation. Evolution here means cumulative change in a reusable repertoire, without a claim of biological evolution or artificial general intelligence.

\section{Related Work}
\paragraph{Feedback and agent memory.}
Reflexion uses linguistic feedback as an experience signal for subsequent attempts, while Self-Refine iterates generation, feedback, and revision without additional model training \citep{shinn2023reflexion,madaan2023selfrefine}. Self-Debugging uses execution and explanations to improve generated programs \citep{chen2024selfdebug}. These systems establish that nonparametric correction can be useful. Our controller separates the implicated scope and the desired effect before proposing a patch, and distinguishes local acceptance from evidence for future reuse. We study that separation in a finite interpreter, not as a replacement implementation of any of these LLM systems.

ExpeL stores extracted insights and past experiences for later decisions \citep{zhao2024expel}. Agent Workflow Memory induces and retrieves routines from previous executions \citep{wang2025awm}. Voyager combines executable skills, environment feedback, and iterative program improvement \citep{wang2024voyager}. EvolveR further couples distilled principles and empirical utility with policy evolution \citep{wu2026evolver}. Consequently, neither reusable experience nor an incentive-like utility score is itself a novelty claim. \PI\ is positioned as a transparent composition of mechanisms whose individual roles can be tested, with a particular emphasis on small interventions and the conditions attached to a retained action.

\paragraph{Explanation and repair.}
Explanation-based generalization connects domain knowledge and an explanation of a specific case to reusable knowledge \citep{mitchell1986ebg}. Our use of typed lineage and contract projection has a clear conceptual debt to that tradition. It is substantially simpler: the experiment does not infer a domain theory, prove a natural-language explanation, or discover a new predicate vocabulary. Delta debugging isolates failure-inducing input through tests \citep{zeller2002dd}. It motivates the importance of narrowing a discrepancy, but our localization is supplied through public data lineage, not computed by a delta-debugging search. Claims about autonomous fault localization would therefore go beyond the experiment.

\paragraph{Imitation and sequential learning.}
Imitation learning must account for the distribution of states induced by a learner's actions; DAgger addresses this through dataset aggregation and no-regret learning \citep{ross2011dagger}. \PI\ does not fit a policy to expert action labels and does not inherit that guarantee. Its evidence is a table of conditioned actions and trial outcomes, with no gradient updates. Temporally extended skills are also central to the options framework \citep{sutton1999options}; our atomic stage transformations are more restricted than general options. These connections encourage precise terminology: this work concerns architectural experience reuse, rather than a new universal imitation-learning objective.

\paragraph{Scope of the empirical comparison.}
All evaluated methods share one interpreter. Whole-trajectory replay is restricted to complete workflows and does not reproduce the broader capabilities of the cited LLM agents. The compositional repair control is the informative comparison for benefits attributable uniquely to \PI; its tie is reported explicitly.

\section{Problem Formulation}
\subsection{Tasks, Programs, and Feedback}
A task is $x=(v,r,\rho,D)$, where $v$ identifies a tool revision, $r$ is a structured request, $\rho$ binds physical fields to logical roles, and $D$ is an input table. A program is a tuple $p=(a_1,\ldots,a_m)$ with $a_j\in\mathcal A_j$. The environment executes
\begin{equation}
z_0=\rho(D),\qquad z_j=T_{v,j,a_j}(z_{j-1}),\quad j=1,\ldots,m.
\end{equation}
The action identifiers are public, but their revision-dependent implementations are not available to the controller as a lookup table. Execution returns an artifact, a before/after trace, and feedback. The public request specifies a local goal $g_j=r_j$. A stage contract $\phi_j(z_{j-1},z_j;g_j)$ states whether that transformation has the requested effect on its actual input.

Let $b_j\in\{0,1\}$ denote the corresponding contract result and let $F(b)=\sum_j(1-b_j)$ count failed contracts. A full audit additionally compares the final aggregate and supporting row identities with an independent reference computation. Its acceptance bit $A(x,p)$ requires both the artifact checks and every stage contract to pass. The controller receives no correct action identifier, hidden registry, or reference answer from the feedback channel. Nevertheless, the feedback is strong: structured diagnostics name a violated output path, and a public lineage map identifies its producer.

This distinction matters for interpreting observability. The controller does not observe latent tool semantics directly, but the benchmark is not a general partially observable decision process with a learned belief state. Its hidden uncertainty is a finite action-to-effect permutation. Exposing logical field roles and contracts removes important sources of difficulty that would occur in a natural document workflow.

\subsection{Experience and the Unit of Reuse}
A source demonstration contains a task, an accepted program, and its stagewise trace. From a transition $(z_{j-1},a_j,z_j)$, an effect recognizer infers $e_j=\psi_j(z_{j-1},z_j)$. A reusable fragment is
\begin{equation}
f=(v,j,e_j,a_j,\text{source identifier}).
\end{equation}
The recognizer uses observable set changes, timestamps, and numerical ratios. It never consults the environment's hidden permutation. In this implementation its effect vocabulary is predefined; abstraction means mapping concrete changes into that vocabulary, not inventing a new concept.

Persistent acquired evidence is indexed by a condition $c=(v,j,g_j)$ and an action $a$. A record $E(c,a)=(s,f,R,U)$ stores successful and failed trial evidence, cumulative incentive, and a set of distinct fully accepted task identifiers. This representation preserves an action's intended role and revision. It does not preserve enough information to certify behavior under arbitrary distribution shift. Source fragments and acquired evidence also have different status: an old fragment may propose a candidate in a new revision, whereas proactive reuse requires confirmations under the current revision key.

\subsection{Budget and Evaluation Objectives}
Each episode permits at most $K$ small trials and two full submissions. The first full submission exposes any discrepancy. After bounded correction, a second submission checks the complete task. A small trial executes a proposed program on a shared subset $\mu(D)$ comprising two complete order histories; the controller changes one stage in that program, rather than executing only the changed stage.

We report final acceptance, first-submission acceptance, the number of small trials, and row-stage visits. The latter counts rows processed by stage operators over full and small executions, plus the rows consumed by the final aggregate/emission step. It excludes retrieval, trace parsing, controller computation, and corpus construction. It is consequently an interpreter-work proxy, not a measurement of tokens, latency, energy, or monetary cost. No fixed long-term memory or context-window constraint is imposed in the present study. Keeping these quantities separate prevents an execution-saving result from being misrepresented as a language-model efficiency result.

\section{Passive Imitation Architecture}
\subsection{From a Discrepancy to a Local Candidate}
Figure~\ref{fig:architecture} illustrates the controller. For a failed diagnostic $d$, localization maps its public output path to a producing stage $j=G(d.\text{path})$. Goal projection then selects $g_j=r_j$. The textual explanation is retained for inspection, but its semantics are not parsed by a language model. The diagnostic path and structured request perform the operative reasoning in this pilot.

Candidate construction first considers actions with positive acquired evidence for $(v,j,g_j)$. It then selects source fragments with the same stage and inferred effect as the local goal, preferring the current revision, followed by a heuristic evidence score and a deterministic source-identifier tie break. All remaining typed actions are appended, with duplicates removed. This final enumeration step makes source demonstrations optional. The controller skips the current program and previously tried pairs of stage and full program signature within the episode.

The candidate transformation is transplanted into the current program only at stage $j$. It is not a complete trajectory replacement. That distinction is the main structural contrast with the whole-trajectory and no-localization controls. A fragment proposes an action that may meet a desired postcondition; it does not provide authority to overwrite unrelated operations or to bypass verification.

\begin{figure*}[t]
\centering
\includegraphics[width=\textwidth]{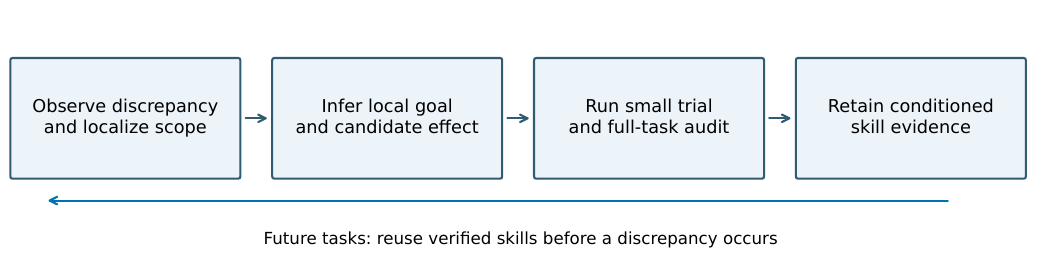}
\caption{The implemented passive-imitation loop. Structured feedback and public lineage identify the local scope; symbolic effect recognition supplies candidate skills. A local trial pass permits a tentative patch. Full-task acceptance contributes evidence for later proactive reuse. Candidate sources may include other executions or the controller's own trials.}
\label{fig:architecture}
\end{figure*}

\subsection{Small Trials and Full Verification}
For proposal $p'=p[j\leftarrow a]$, the environment returns microtask feedback $b'$. The patch is tentatively committed when the target contract passes. If it fails, the current program is retained and the next distinct candidate is tried. The controller always chooses the first currently failed diagnostic in the fixed stage order. It stops when all current contracts pass, the trial budget is exhausted, or no eligible proposal remains. A final full-task audit is performed after any initially failed submission, even if microtask feedback appears satisfactory.

This design separates hypothesis testing from promotion. A small test can reject a candidate cheaply, and in the constructed vocabulary it identifies each valid transformation. Outside that vocabulary it is only evidence on a subset. The full audit therefore remains part of the controller. Its role cannot be replaced by a positive incentive or a plausible explanation. In the implementation, tentative commitment depends on the target contract, not on the sign of a reward.

\begin{algorithm}[t]
\caption{Passive Imitation Episode}
\label{alg:pi}
\begin{algorithmic}[1]
\REQUIRE Task $x$, persistent evidence $E$, fragments $\mathcal F$, budget $K$, acquisition flag $L$
\STATE $p\gets\operatorname{Initialize}(x,E)$; $y\gets\operatorname{FullAudit}(x,p)$
\STATE $b\gets y.\text{contracts}$; $\mathcal H\gets\emptyset$; $\mathcal U\gets[\,]$
\IF{$y.\text{accepted}$}
  \IF{$L$} \STATE Record full-task confirmation for $p$ in $E$ \ENDIF
  \STATE \textbf{return} $(p,y,E)$
\ENDIF
\FOR{$t=1$ to $K$}
  \IF{all contracts in $b$ pass} \STATE \textbf{break} \ENDIF
  \STATE $d\gets$ first violated diagnostic; $j\gets G(d.\text{path})$; $g\gets r_j$
  \STATE $a\gets$ next candidate from $E,\mathcal F,\mathcal A_j$, excluding repeats and current program
  \IF{no candidate exists} \STATE \textbf{break} \ENDIF
  \STATE $p'\gets p[j\leftarrow a]$; add $(j,p')$ to $\mathcal H$
  \STATE $y'\gets\operatorname{SmallTrial}(\mu(x),p')$; compute $r_t$ by Eq.~\ref{eq:reward}
  \STATE Append $(j,a,\text{target pass},r_t)$ to $\mathcal U$
  \IF{the target contract passes in $y'$}
     \STATE $p\gets p'$; $b\gets y'.\text{contracts}$
  \ENDIF
\ENDFOR
\STATE $y\gets\operatorname{FullAudit}(x,p)$
\IF{$L$}
  \STATE Update $E$ from $\mathcal U$; if $y$ is accepted, add distinct full-task confirmations for $p$
\ENDIF
\STATE \textbf{return} $(p,y,E)$
\end{algorithmic}
\end{algorithm}

\subsection{Incentive, Retention, and Proactive Reuse}
For a microtask with $n_\mu$ input rows and execution work $W_t$, the recorded incentive is
\begin{equation}
r_t=F(b)-F(b')-0.01\frac{W_t}{\max(1,n_\mu)}.
\label{eq:reward}
\end{equation}
If the target contract fails, the stored value is $\min(0,r_t)$. Successful and failed target trials increment their corresponding evidence counts during acquisition. A fully accepted task supplies one additional success and one unit of incentive to each action in its final program; its identifier is added to $U$ without duplication. Full-task rejection withholds those confirmations. It does not itself add a failure count to every stage.

The reuse score is
\begin{equation}
q(c,a)=\frac{s(c,a)+1}{s(c,a)+f(c,a)+2}.
\label{eq:score}
\end{equation}
An action may enter the initial program of a later task when $|U(c,a)|\geq2$ and $q(c,a)\geq0.7$. Eligible actions are ordered by $q$, with cumulative incentive as a tie break. The score has a smoothed-count form, but we do not interpret it as a calibrated posterior probability: microtask outcomes and full-task confirmations can be correlated and can arise within the same episode. The two-task requirement counts distinct task identifiers, not statistically independent proofs of reliability.

At frozen evaluation time, within-episode repair is permitted, but $E$ is not updated. Acquisition and evaluation are separate calls. This boundary is enforced by state hashes before and after each frozen batch. Resetting the retained evidence restores the cold controller. The architecture can therefore improve future first submissions through explicit experience accumulation, rather than through hidden changes during testing.

\section{Conditional Analysis}
\subsection{When Local Repair Has a Finite Query Bound}
The following analysis describes a finite interface, not unrestricted software repair. Let $J$ be the stages failing in the initial program. Three assumptions specify when local evidence is sufficient.

\begin{assumption}[Faithful Localization]
Every violated stage contract identifies its producing stage through $G$, and every incorrect stage is exposed by a violated contract. The diagnostic order visits an exposed incorrect stage until it is repaired.
\end{assumption}
\begin{assumption}[Stable Identifiability]
For each condition $(v,j,g)$ there is a unique correct action. For all inputs reachable during repair, that action passes the stage contract and every other action fails it. The shared microtask and full task agree on these contract outcomes.
\end{assumption}
\begin{assumption}[Complete Local Search]
Candidates cover every alternative action in $\mathcal A_j$. Trials alter only the selected action, and failed candidates are not retried at that stage before a successful repair.
\end{assumption}

\begin{proposition}[Finite Local Repair]
Under Assumptions 1--3, a correct program is found using at most
\begin{equation}
Q\leq\sum_{j\in J}(|\mathcal A_j|-1)
\label{eq:bound}
\end{equation}
small trials and two full submissions, provided $K\geq Q$.
\end{proposition}
\begin{proof}
At an incorrect stage, the current action is already known to fail. Exhausting its alternatives requires at most $|\mathcal A_j|-1$ queries and includes the unique correct action. Identifiability ensures that only this action is committed. Locality leaves other action choices unchanged, and stability ensures that a later change in their inputs cannot invalidate a correct action. Thus a repaired stage does not become incorrect again. Summing over initially incorrect stages proves the bound. The initial and final full submissions expose the initial discrepancy and verify the repaired complete program.
\end{proof}

For the present action-set sizes $(2,3,3)$, the worst-case bound is five small trials. The empirical budget of six therefore exceeds this bound. A budget of three need not suffice when candidate ordering is uninformed. On familiar revisions, effect-matched candidates often place the correct action first, reducing the number of needed queries to the number of incorrect stages. The cold composition result is consistent with this mechanism, rather than evidence of solving a large unconstrained search problem.

The strong assumptions explain the limits of the guarantee. A contract on a final aggregate alone may not expose which stage is wrong. A tiny fixture may fail to distinguish actions on rare cases. An upstream patch may change the meaning of a downstream predicate. In any of these situations the proof no longer applies. Our exhaustive checks verify micro/full agreement for all 18 programs on every generated primary task; they do not establish stable identifiability for programs or inputs outside this generated set.

\subsection{Why an Incentive Threshold Can Be Redundant}
\begin{proposition}[Redundant Promotion Threshold]
Suppose accepted tasks certify the unique correct action for a condition, and that action never receives failed local evidence. After two distinct full-task confirmations its score in Eq.~\ref{eq:score} is at least $3/4$. A threshold of $0.7$ therefore imposes no additional promotion delay relative to the two-task condition.
\end{proposition}
\begin{proof}
Two confirmations imply $s\geq2$ and the premise gives $f=0$. Hence $(s+1)/(s+2)\geq3/4>0.7$. Uniqueness prevents a second incorrect action from receiving valid full-task confirmations for the same condition.
\end{proof}

This proposition is about eligibility, not a universal equivalence between two controllers. Candidate ordering before promotion can still differ, and noisy tests can invalidate the premise. It nevertheless explains why the numerical threshold has little opportunity to help in a deterministic world with perfect diagnostic separation. The strong control removes that threshold and removes evidence-score ranking within source fragments, but shares the remaining controller machinery. The reported tie is therefore plausible and substantively important.

\subsection{What Verification Does Not Guarantee}
Two accepted tasks do not certify distribution-free reliability. An action can pass sampled inputs and fail elsewhere. Adaptive selection and correlated tests also preclude simply multiplying per-test error probabilities. The evidence score therefore supplies no calibrated probability of future correctness.

Version keys prevent automatic promotion across versions but fragment experience when versions share semantics. Cross-version fragments remain hypotheses requiring tests. This pilot measures consequences of that reuse--caution tradeoff without optimizing applicability conditions.

\section{Experimental Design}
\subsection{A Finite Reporting Environment}
The environment has three ordered stages. Population selection either keeps all records or only settled records. Revision resolution keeps all revisions or chooses the oldest or latest revision per order. Amount normalization interprets source values as yuan, cents, or hundreds of yuan. Actions have opaque identifiers $\{F0,F1\}$, $\{D0,D1,D2\}$, and $\{N0,N1,N2\}$. Each simulated world independently permutes the meaning of these identifiers for revisions A, B, and C.

Ordinary tasks contain six orders with two revisions each. Larger schema-shift tasks contain eight orders with three revisions each and allow settlement status to change across revisions. Column names vary, but a public role mapping is always supplied. Schema shift thus tests execution under changed layouts and row structure with known roles; it does not test autonomous schema understanding. Every task requests an exact final total and the identities of its supporting rows.

The source corpus contains 24 accepted demonstrations per world: two familiar revisions, six atomic/default requests, and two separately generated examples for each pair. A default request asks for all records, all revisions, and yuan. Each of the other five atomic requests changes one of these goals. Source transitions are converted into 16 distinct revision-conditioned fragments. Complete requests with at least two nondefault goals are withheld from both demonstrations and the initial acquisition stream. There are 12 such combinations.

The microtask contains the first two order identifiers and all their revisions: four rows ordinarily and six rows in the larger tasks. This is a deliberately identifying fixture for the finite operator vocabulary. It is shared by all controllers, is not selected using a hidden correct-action label, and is checked exhaustively against full-task contract results. The ease of constructing it is part of the benchmark design, not an autonomous experimental-design achievement.

\subsection{Primary Protocol and Controls}
We fix ten seeds: 11, 29, 47, 61, 83, 101, 127, 149, 173, and 197. Each seed defines one tool world. Acquisition supplies 96 atomic tasks, consisting of eight instances of each familiar revision/request pair. Final evaluation has 240 tasks per world, evenly divided among new atomic instances, held-out compositions, larger schema-shift compositions, and unseen-revision-C compositions. All controllers see the same tasks and public information. The local protocol was recorded before the primary run; it was not publicly preregistered.

Cold controllers are evaluated at $K=3$. After acquisition, frozen controllers are evaluated at $K\in\{1,3,6\}$. Separate 24-task frozen probes are reused at acquisition checkpoints 0, 24, 48, and 96. The final task identifiers are disjoint from acquisition, source, and probe identifiers. Reusing probes across checkpoints produces repeated measurements, not independent samples. A reset condition discards acquired evidence and reevaluates \PI.

Eight controllers isolate the relevant operations. \emph{Whole-trajectory replay} proposes only complete source programs, retains accepted complete workflows by exact request and revision, and tests proposals with the same microtask budget. \emph{Full \PI} implements Algorithm~\ref{alg:pi}. \emph{No localization} uses the diagnosed goal to rank complete templates but replaces the whole program. \emph{No goal matching} preserves local scope while ranking raw templates by request similarity instead of matching inferred transition effects. \emph{No small trial} directly adopts the highest-ranked local candidates and proceeds to the second full audit. \emph{No retention}, named \texttt{pi\_no\_incentive} in code, disables all persistent learned-action evidence. \emph{Compositional repair} preserves local goal-directed testing and two-task confirmation while removing the additional score threshold and source-fragment evidence ranking. \emph{Self-trial only} enumerates typed actions without consulting demonstrations, while preserving localization, verification, and learning.

The no-retention condition is not a clean ablation of a reward formula: it removes cross-task verified reuse altogether. The strong compositional control is needed to assess whether the extra numerical machinery contributes beyond local repair and confirmation. It still shares confidence-based tie ordering in some helper functions; it is a targeted control, not an independently engineered competing system. These details are exposed in the implementation appendix.

\subsection{Measurement and Integrity}
Final success is the fraction of tasks accepted by the independent full evaluator. First success measures acceptance before current-episode correction. Final collateral regression occurs when a contract passing initially fails at the final audit. This differs from the number of regressions across tentative commits, which is logged separately. Mean row-stage visits includes unsuccessful attempts and trial work; lower cost is not automatically preferable if success also falls.

We aggregate 60 final tasks within each seed/group, then summarize the ten seed values. Since group sizes are equal, the mean equals the pooled rate across 600 tasks. Where reported, standard deviations describe variability across the ten worlds, not an independent-task confidence interval. Common tasks make method comparisons paired; task counts must not be treated as 600 independent tool environments. We use descriptive comparisons without selecting a statistical significance threshold.

Before score interpretation, the integrity suite checks 3,600 generated tasks and all 18 programs per task, for 64,800 full/micro comparisons. It verifies a unique accepting program, the rejection of corrupted totals and provenance, source-effect inference, disjoint identifiers, intervention scope, budget limits, and frozen-state hashes. After the first run, a logging-only omission in the trajectory control's tentative-regression counter was corrected. The experiment was rerun; every other summary field was verified unchanged. The public release was rerun again after changing only output paths, and all eight deterministic summary/check files were byte-identical to the corrected original outputs.

\section{Results}
\subsection{Local Repair Enables Withheld Compositions}
Table~\ref{tab:main} reports cold composition results at $K=3$. \PI\ reaches 100\% final success, using 1.74 small trials on average. Whole-trajectory replay remains at 6.17\%, the same rate as its initial default program. This large gap has a specific interpretation: none of the observed atomic complete programs is an exact solution to a genuinely multi-goal familiar-revision request. The replay controller cannot assemble missing combinations. Its residual successes arise from the default opaque identifiers happening to implement the requested combination in a particular world.

Removing localization reduces success to 9.50\% and produces 41.83\% final collateral regression. A complete replacement can satisfy the currently targeted goal while breaking a previously correct stage. Locality prevents that behavior in this finite world. Removing effect-to-goal matching has a smaller effect: success falls to 92.50\%, with 1.90 trials. Raw request similarity still contains useful information, so this condition does not remove all goal information or all reasoning about the request.

\begin{table*}[t]
\caption{Primary results at $K=3$. Cold composition columns use no acquired action evidence. Learned first success follows 96 atomic acquisition tasks. Unseen C is evaluated with frozen A/B-acquired evidence. Percentages are means over ten worlds, 60 tasks per group per world. Parentheses in the last column give the across-world standard deviation in percentage points. ``No retention'' corresponds to \texttt{pi\_no\_incentive}.}
\label{tab:main}
\centering
\begin{tabular}{lrrrrr}
\toprule
&\multicolumn{3}{c}{Cold held-out composition}&Learned composition&Unseen C\\
\cmidrule(lr){2-4}
Controller&Final success&Small trials&Collateral&First success&Final success (SD)\\
\midrule
Whole-trajectory replay&6.17&2.82&12.83&6.17&22.67 (12.02)\\
Full \PI&100.00&1.74&0.00&100.00&70.00 (21.20)\\
No localization&9.50&2.82&41.83&100.00&19.17 (10.10)\\
No goal matching&92.50&1.90&0.00&100.00&76.00 (19.57)\\
No small trial&100.00&0.00&0.00&100.00&24.00 (14.93)\\
No retention&100.00&1.74&0.00&6.17&70.00 (21.20)\\
Compositional repair&100.00&1.74&0.00&100.00&70.00 (21.20)\\
Self-trial only&80.33&2.11&0.00&100.00&79.00 (9.91)\\
\bottomrule
\end{tabular}
\end{table*}

The no-small-trial condition also solves all cold familiar compositions. Same-revision effect-matched fragments already propose correct local actions, making validation redundant for these tasks. Its work is 58.82 row-stage visits versus 75.90 for \PI. A claim that small trials always improve efficiency would therefore contradict the measured behavior. Their benefit must be assessed in environments where proposal reliability is lower.

The strong compositional repair control matches \PI\ in success, trials, work, and collateral regression in the reported settings. This comparison prevents attributing the composition result uniquely to the incentive score. Self-trial enumeration, which uses no demonstrations, achieves 80.33\% cold success and also has zero collateral regression. Demonstrations improve candidate ordering on familiar semantics, but are not a necessary prerequisite for correction or subsequent learning.

\subsection{Retention Changes Future First Submissions}
After 96 atomic acquisition tasks, \PI\ solves every held-out composition on the first submission. The same is true of the strong control and self-trial controller. They have acquired componentwise associations that can be recombined into complete requests never observed in acquisition. The no-localization and no-goal variants also reach 100\% first success after this repeated atomic stream; their cold repair weaknesses do not prevent eventual acquisition in the simpler training regime.

Disabling retention preserves final success at 100\% on familiar compositions but leaves first success at 6.17\%. It repeatedly performs the same 1.74-trial correction. Mean evaluation work is 75.90 row-stage visits without retention and 28.98 with acquired \PI\ skills, a reduction of 61.82\%. This is an evaluation-phase work reduction. Acquisition, demonstrations, and retrieval costs are not amortized into that percentage. It does not establish an end-to-end token saving or a favorable break-even point for an LLM implementation.

Whole-trajectory replay remains at 6.17\% on withheld compositions, as expected from its restricted representation. Local verified associations instead transfer across recombinations while preserving the frozen-state boundary; this comparison does not establish a failure of modern memory-based agents.

\subsection{New Tool Semantics Expose Negative Transfer}
Revision C permutes action meanings independently of A and B. Acquired records use version-conditioned keys, so they do not authorize proactive reuse in C. Retrieved source fragments nevertheless influence the order in which opaque actions are tested. At $K=3$, \PI\ obtains 70.0\% final success compared with 24.0\% for direct adoption without small trials. Validation and local search repair many proposals whose old meanings are no longer reliable.

However, self-trial enumeration reaches 79.0\%, and removing goal-based fragment matching reaches 76.0\%. Thus the source-induced ordering is not uniformly useful. An action identifier associated with the desired effect in another revision can consume a scarce trial without carrying predictive information about C. Across-world variation is substantial: the standard deviation is 21.20 percentage points for \PI, versus 9.91 for self-trial enumeration. A pooled mean alone hides this variation in transfer difficulty.

Figure~\ref{fig:adaptation}, left, shows the budget dependence. At $K=1$, \PI\ reaches 21.5\%, below the 24.0\% no-trial condition. At $K=6$, \PI, compositional repair, and self-trial enumeration all reach 100\%, consistent with the finite-search bound. Self-trial still uses fewer trials on average in this setting: 2.46 versus 2.67 for \PI. These results identify candidate ordering as a meaningful failure point rather than supporting an unconditional advantage for imitation.

\begin{figure*}[t]
\centering
\includegraphics[width=\textwidth]{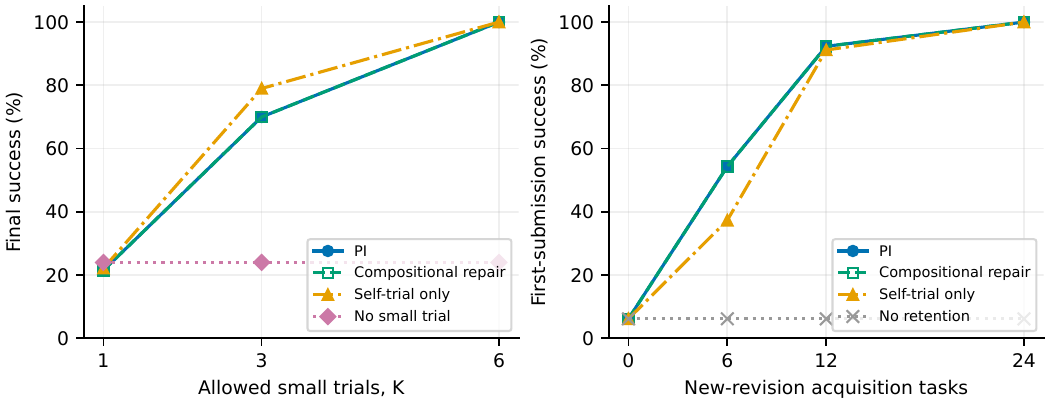}
\caption{Left: final success on previously unseen revision C after A/B acquisition, with persistent state frozen. The no-trial controller does not use $K$. Right: first-submission success after separate C acquisition tasks in the exploratory follow-up. Lines show means over ten worlds; PI and compositional repair overlap. Repeated evaluations use the same frozen final C tasks. Learning C changes the interpretation from unseen-revision transfer to retention within a learned revision.}
\label{fig:adaptation}
\end{figure*}

\subsection{Exploratory Adaptation to the New Revision}
The primary frozen experiment cannot measure learning across C tasks because its persistent state is intentionally fixed. After examining those results, we therefore declared a separate exploratory follow-up. It preserves the original seeds, tools, final C tasks, and $K=3$. Each controller first rebuilds its A/B history, then acquires experience on 24 separately generated C tasks: two instances of each multi-goal request, shuffled independently. Their task identifiers and rows are separate from the final evaluation tasks. Four controllers are evaluated at 0, 6, 12, and 24 C exposures.

For \PI, first-submission success rises from 6.17\% to 54.33\%, 92.33\%, and 100\%. Final success rises from 70.0\% to 96.5\%, 100\%, and 100\%. Mean work falls from 82.19 to 49.40, 31.47, and 28.93 row-stage visits. The strong compositional control follows the same curve. Self-trial reaches 37.33\%, 91.17\%, and 100\% first success at the three nonzero checkpoints. No retention remains at 6.17\% first success and 70.0\% final success throughout.

These results demonstrate accumulation of verified mappings in a newly encountered but subsequently learned revision. They do not show zero-shot transfer after 24 exposures, nor do they isolate a benefit from numerical incentives over simpler verified reuse. Because this follow-up was designed after the primary analysis, it is reported as exploratory instead of being presented as a prespecified confirmatory experiment.

\subsection{An Inspectable Correction Trace}
One archived composition task in seed 11 begins with $(F0,D0,N0)$ and violates all three contracts. The feedback identifies the selected-record path; goal projection requests all records. A fragment proposes $F1$, and the microtask confirms that local effect. The remaining discrepancies then lead to $D1$ for oldest revisions and $N2$ for values denominated in hundreds of yuan. The recorded incentives are 0.96, 0.97, and 0.97 as the failed-contract count falls from three to zero. A full-task audit accepts the final program.

The trace exposes each supplied mechanism: path-based localization, symbolic goal projection, programmed effect recognition, trial evidence, and full confirmation. Acquisition updates future preferences; frozen testing does not. The complete example and implementation are released.

\section{Discussion and Limitations}
\subsection{What the Study Establishes}
The study supports local repair and conditioned retention when interfaces expose the relevant factors. The query analysis explains their efficiency and the strong control's identical results. It does not establish an advantage for a unique incentive mechanism.

Perfectly identifying tests make promotion scores largely redundant. The no-retention ablation tests retained knowledge rather than reward shape, while negative transfer shows that an effect label does not justify reusing an opaque action across versions.

Self-improvement here means that earlier accepted tasks modify persistent state and improve later first submissions. It does not mean discovering new algorithms, expanding neural reasoning capacity, or adapting to arbitrary goal and tool changes.

\subsection{Threats to External Validity}
First, this is one synthetic domain with 18 possible complete programs and predefined stage boundaries. The environment was designed to make composition informative. A whole-trajectory-only controller consequently has a representational disadvantage. The results cannot be generalized to all retrieval methods, much less to existing LLM agents with unrestricted planning and code generation.

Second, localization and abstraction receive unusually informative interfaces. Public role bindings, typed lineage, and stage assertions make scope selection inexpensive and largely deterministic. Natural tasks may expose only a vague complaint or a final incorrect answer. The ability to infer a correct local goal from that feedback is not tested here. In particular, no claim about detecting or resolving natural-language hallucinations follows from these experiments.

Third, microtests are diagnostic by construction. Their agreement with the full task is exhaustively verified in the finite generated world, but their reliability would be unknown for an arbitrary program, external API, or changing dataset. The full evaluator also has access to a reference computation that may be unavailable in practice. Plausible-looking feedback and a positive incentive cannot substitute for that missing verifier.

Fourth, the study does not constrain persistent storage, bound prompt length, or measure language-model tokens. The evidence table retains task identifiers, and its size can grow with the acquisition stream. A bounded-memory implementation would need to decide which conditions, failures, and confirmations can be merged or discarded. Our measurements establish only that learned reuse can reduce repeated interpreter work in evaluation.

Finally, the thresholds, cost coefficient, candidate-ordering rules, and fixture design are not validated across a broad family of environments. We report fixed seeds and preserve the original protocol, but do not claim public preregistration or a comprehensive hyperparameter study. The online follow-up is exploratory, and the ten tool worlds share the same hand-designed task generator. Both facts limit the strength of extrapolation.

\subsection{Implications for Architecture Research}
A stronger test would require scope, goal, or applicability inference under imperfect feedback, with independently measured outcomes. It should retain an equally expressive repair-and-reuse control and a demonstration-free condition, since modularity and enumeration already explain substantial present benefits.

An LLM implementation should charge diagnosis, candidate generation, trial interpretation, memory maintenance, and failed attempts to a matched total budget. Token-saving claims require measured totals and acquisition-inclusive success--cost comparisons; those results are outside this study.

\section{Conclusion}
Passive imitation turns structured discrepancies into local repair hypotheses and accumulates condition-dependent verified skills. In a finite reporting environment, locality enables withheld compositions and retention improves future first submissions. Trials help unreliable proposals, whereas familiar fragments need no trials and source ordering can hurt under semantic changes. A compositional control matches PI, supporting the shared mechanism without a unique incentive advantage. Released code, exhaustive checks, and negative findings provide a basis for testing learned, uncertain interfaces.
\label{mainbodyend}

\clearpage
\section*{Impact Statement}
This work studies a small synthetic controller and makes no deployment claim. Verified local reuse could reduce repeated mistakes, but incorrect applicability conditions can propagate failures across tasks. In consequential settings, small tests must not be interpreted as guarantees, and retained behavior must remain subject to appropriate full-task checks. All experimental records are synthetically generated; no personal documents, user feedback, or human-subject data are used. Public code and explicit reporting of baseline ties and transfer failures are intended to make unsupported extrapolations easier to identify.

\bibliography{references}
\bibliographystyle{icml2025}

\clearpage
\onecolumn
\appendix
\section{Implementation and Reproduction}
This appendix contains the complete three-script implementation used for the public portable release. Python 3.10 or newer is sufficient; the executed release used Python 3.12. No neural model, GPU, third-party Python package, or network connection is required for the experiment itself. Run \texttt{verify\_correction.py}, then \texttt{run\_correction.py}, then \texttt{run\_online\_adaptation.py}. Outputs are written to a \texttt{results} directory beside the scripts. Run the primary script before the follow-up so that this directory exists.

The main run executes 94,595 episodes: 79,200 final/reset evaluations, 7,680 acquisition episodes, 7,680 checkpoint evaluations, and 35 detailed examples. The follow-up executes 14,400 episodes, including reconstruction of A/B histories. These totals include repeated evaluation conditions and are not counts of independent tasks. The portable rerun took approximately 21.93 seconds for the primary experiment and 3.49 seconds for the follow-up on the local CPU; these unprofiled measurements are supplied for practical reproduction only, not as comparative runtime results. Integrity checks run separately.

The output directory adjustment is the only change between the corrected original controller and the public release. Eight deterministic result/check files were compared byte for byte. Execution-time metadata and the script checksum are expected to change across releases or runs. Source hashes and the result comparison are recorded in \texttt{REPRODUCIBILITY.json} at the repository. The archived local protocol predates the primary run but was not externally timestamped as a preregistration. Its initial description of compositional repair as a simple lookup should be read with the exact shared helper functions below: the control removes the promotion score threshold and source-fragment score ranking, while retaining shared ordering in other paths.

\subsection{Environment, Controllers, and Primary Experiment}
\begin{lstlisting}
"""CPU-only PI correction-loop benchmark. All scores arise from execution.

Policy inputs contain public requests, anonymous tool action IDs, observed
traces and assertion explanations. Hidden registries are used only by tools,
demonstrators and the independent evaluator, never read by policy code.
"""
from __future__ import annotations
import copy
import csv
import hashlib
import itertools
import json
import random
import statistics
import time
from collections import defaultdict
from dataclasses import asdict, dataclass
from pathlib import Path

ROOT = Path(__file__).resolve().parent
OUT = ROOT / 'results'
STAGES = ('population', 'revision', 'amount')
ACTION_IDS = {'population': ('F0', 'F1'), 'revision': ('D0', 'D1', 'D2'),
              'amount': ('N0', 'N1', 'N2')}
EFFECTS = {'population': ('all', 'settled'), 'revision': ('all', 'oldest', 'latest'),
           'amount': ('yuan', 'cent', 'hundred_yuan')}
DEFAULT_GOALS = {'population': 'all', 'revision': 'all', 'amount': 'yuan'}
DEFAULT_PROGRAM = {stage: ACTION_IDS[stage][0] for stage in STAGES}
SEEDS = [11, 29, 47, 61, 83, 101, 127, 149, 173, 197]
METHODS = ['trajectory_replay', 'pi_full', 'pi_no_localization', 'pi_no_goal',
           'pi_no_trial', 'pi_no_incentive', 'compositional_repair', 'self_trial_only']
ASSERTIONS = {
    'population_constraint': {'path': 'selected_record_ids',
        'explanation': "Included records violate the report's requested population."},
    'revision_constraint': {'path': 'resolved_record_ids',
        'explanation': 'Chosen revisions violate the requested identity/revision rule.'},
    'value_constraint': {'path': 'normalized_amounts',
        'explanation': 'Displayed amounts do not preserve value under the declared source denomination.'}}
# Public typed pipeline data lineage, not generator-provided fault locations.
LINEAGE = {'selected_record_ids': 'population', 'resolved_record_ids': 'revision',
           'normalized_amounts': 'amount'}


def stable_int(value):
    return int(hashlib.sha256(value.encode()).hexdigest()[:16], 16)


@dataclass
class Task:
    task_id: str
    group: str
    tool_revision: str
    request: dict
    roles: dict
    rows: list


def make_task(seed, task_id, group, revision, request, shifted=False):
    r = random.Random(stable_int(f'{seed}|{task_id}|rows'))
    roles = {name: f'{name}_{r.randrange(10**9):09d}' for name in
             ('amount', 'order', 'settled', 'updated', 'row_id')}
    rows = []
    for order in range(8 if shifted else 6):
        for version in range(3 if shifted else 2):
            rows.append({roles['amount']: r.randrange(100, 20000) + version,
                         roles['order']: f'order{order}',
                         roles['settled']: (order + version) % 2 == 0 if shifted else order % 2 == 0,
                         roles['updated']: version,
                         roles['row_id']: f'{task_id}:order{order}:v{version}'})
    r.shuffle(rows)
    return Task(task_id, group, revision, dict(request), roles, rows)


def atomic_requests():
    requests = [dict(DEFAULT_GOALS)]
    for stage in STAGES:
        for goal in EFFECTS[stage]:
            if goal != DEFAULT_GOALS[stage]:
                request = dict(DEFAULT_GOALS); request[stage] = goal; requests.append(request)
    assert len(requests) == 6
    return requests


def acquisition_tasks(seed):
    specs = [(revision, request) for revision in ('A', 'B') for request in atomic_requests()] * 8
    random.Random(seed + 1000).shuffle(specs)
    return [make_task(seed, f'learn-{i:03}', 'atomic_acquisition', rev, req)
            for i, (rev, req) in enumerate(specs)]


def evaluation_tasks(seed, prefix='test', size=240):
    r = random.Random(seed + (3000 if prefix == 'test' else 4000))
    compositions = [dict(zip(STAGES, values)) for values in itertools.product(*(EFFECTS[s] for s in STAGES))
                    if sum(a != DEFAULT_GOALS[s] for s, a in zip(STAGES, values)) >= 2]
    tasks = []
    for i in range(size):
        group = ('atomic_new', 'heldout_composition', 'schema_shift', 'new_tool_revision')[i // (size // 4)]
        rev = 'C' if group == 'new_tool_revision' else r.choice(('A', 'B'))
        request = r.choice(atomic_requests() if group == 'atomic_new' else compositions)
        tasks.append(make_task(seed, f'{prefix}-{i:03}', group, rev, request, shifted=group == 'schema_shift'))
    return tasks


def micro_task(task):
    """Same public microcase for all methods; two complete order histories."""
    order = task.roles['order']
    chosen = sorted({row[order] for row in task.rows})[:2]
    return Task(task.task_id + ':micro', task.group, task.tool_revision,
                dict(task.request), dict(task.roles),
                [dict(row) for row in task.rows if row[order] in chosen])


def observed_rows(task, rows):
    """Public adapter: bind varying field names to typed logical roles."""
    return [{role: row[column] for role, column in task.roles.items()} for row in rows]


def transform(rows, stage, effect):
    """Implementation of the environment's typed anonymous tool actions."""
    if stage == 'population':
        return [dict(row) for row in rows if effect == 'all' or row['settled']]
    if stage == 'revision':
        if effect == 'all': return [dict(row) for row in rows]
        ordered = sorted(rows, key=lambda row: row['updated'], reverse=effect == 'latest')
        seen, result = set(), []
        for row in ordered:
            if row['order'] not in seen: result.append(dict(row)); seen.add(row['order'])
        return result
    divisor = {'yuan': 1, 'cent': 100, 'hundred_yuan': .01}[effect]
    return [dict(row, amount=round(row['amount'] / divisor, 2)) for row in rows]


def canonical(rows):
    return sorted(rows, key=lambda row: row['row_id'])


class Environment:
    def __init__(self, seed):
        rng = random.Random(seed)
        self.registry = {}
        for revision in ('A', 'B', 'C'):
            self.registry[revision] = {}
            for stage in STAGES:
                values = list(EFFECTS[stage]); rng.shuffle(values)
                self.registry[revision][stage] = dict(zip(ACTION_IDS[stage], values))

    def expert_program(self, task):
        # Only demonstration generation and integrity tests call this method.
        return {s: next(action for action, effect in self.registry[task.tool_revision][s].items()
                        if effect == task.request[s]) for s in STAGES}

    def execute(self, task, program):
        rows = observed_rows(task, task.rows)
        trace, work = [], 0
        for stage in STAGES:
            before = rows
            effect = self.registry[task.tool_revision][stage][program[stage]]
            rows = transform(before, stage, effect)
            work += len(before)
            trace.append({'stage': stage, 'action': program[stage], 'before': before, 'after': rows})
        work += len(rows)  # aggregate / emit
        artifact = {'total_yuan': round(sum(row['amount'] for row in rows), 2),
                    'source_rows': sorted(row['row_id'] for row in rows)}
        return artifact, trace, work

    def independent_goal_rows(self, rows, stage, goal):
        """Separate reference implementation, not transform()."""
        if stage == 'population':
            return [dict(row) for row in rows if goal != 'settled' or bool(row['settled'])]
        if stage == 'revision':
            if goal == 'all': return [dict(row) for row in rows]
            buckets = defaultdict(list)
            for row in rows: buckets[row['order']].append(row)
            chooser = min if goal == 'oldest' else max
            return [dict(chooser(bucket, key=lambda x: x['updated'])) for bucket in buckets.values()]
        factor = {'yuan': 100, 'cent': 1, 'hundred_yuan': 10000}[goal]
        return [dict(row, amount=round(row['amount'] * factor / 100, 2)) for row in rows]

    def feedback(self, task, artifact, trace):
        passed = {}
        diagnostics = []
        for assertion, entry in zip(ASSERTIONS, trace):
            stage = entry['stage']
            wanted = self.independent_goal_rows(entry['before'], stage, task.request[stage])
            passed[assertion] = canonical(wanted) == canonical(entry['after'])
            if not passed[assertion]: diagnostics.append({'assertion': assertion, **ASSERTIONS[assertion]})
        reference = observed_rows(task, task.rows)
        for stage in STAGES:
            reference = self.independent_goal_rows(reference, stage, task.request[stage])
        expected = {'total_yuan': round(sum(row['amount'] for row in reference), 2),
                    'source_rows': sorted(row['row_id'] for row in reference)}
        accepted = (abs(artifact['total_yuan'] - expected['total_yuan']) < 1e-6
                    and artifact['source_rows'] == expected['source_rows'] and all(passed.values()))
        # Numeric answer / reference rows / registry are never in returned feedback.
        return {'accepted': accepted, 'passed': passed, 'diagnostics': diagnostics}

    def submit(self, task, program):
        artifact, trace, work = self.execute(task, program)
        return {'artifact': artifact, 'trace': trace, 'work': work,
                'feedback': self.feedback(task, artifact, trace)}


def demonstrations(environment, seed):
    corpus = []
    for revision in ('A', 'B'):
        for index, request in enumerate(atomic_requests()):
            for repeat in range(2):
                task = make_task(seed, f'source-{revision}-{index}-{repeat}',
                                 'external_atomic', revision, request)
                program = environment.expert_program(task)
                result = environment.submit(task, program)
                assert result['feedback']['accepted']
                corpus.append({'demo_id': task.task_id, 'task': asdict(task),
                               'program': program, 'trace': result['trace'],
                               'accepted': True})
    random.Random(seed + 2000).shuffle(corpus)
    return corpus


def infer_effect(entry):
    """Classify observed state changes. No registry or supplied correct-effect tag."""
    before, after = entry['before'], entry['after']
    before_ids = {row['row_id'] for row in before}; after_ids = {row['row_id'] for row in after}
    if entry['stage'] == 'population':
        if before_ids == after_ids: return 'all'
        if after_ids == {row['row_id'] for row in before if row['settled']}: return 'settled'
    elif entry['stage'] == 'revision':
        if before_ids == after_ids: return 'all'
        groups = defaultdict(list)
        for row in before: groups[row['order']].append(row)
        for label, choose in (('oldest', min), ('latest', max)):
            chosen = {choose(group, key=lambda x: x['updated'])['row_id'] for group in groups.values()}
            if chosen == after_ids: return label
    else:
        originals = {row['row_id']: row['amount'] for row in before}
        ratios = {round(row['amount'] / originals[row['row_id']], 8) for row in after}
        if len(ratios) == 1:
            return {1: 'yuan', .01: 'cent', 100: 'hundred_yuan'}.get(next(iter(ratios)))
    return None


def diagnosed_goals(task, feedback):
    """Symbolic explanation -> public producer scope -> requested postcondition."""
    result = []
    for diagnostic in feedback['diagnostics']:
        stage = LINEAGE[diagnostic['path']]
        result.append({'assertion': diagnostic['assertion'], 'stage': stage,
                       'goal': task.request[stage], 'explanation': diagnostic['explanation']})
    return result


def failed_count(feedback):
    return sum(not value for value in feedback['passed'].values())


class Controller:
    def __init__(self, method, corpus):
        self.method = method
        self.corpus = copy.deepcopy(corpus)
        self.templates = []
        self.fragments = []
        for demo in corpus:
            template = {'revision': demo['task']['tool_revision'], 'request': demo['task']['request'],
                        'program': demo['program'], 'demo_id': demo['demo_id']}
            if not any(t['revision'] == template['revision'] and t['request'] == template['request'] for t in self.templates):
                self.templates.append(template)
            for entry in demo['trace']:
                effect = infer_effect(entry)
                if effect is None: continue
                fragment = {'revision': demo['task']['tool_revision'], 'stage': entry['stage'],
                            'effect': effect, 'action': entry['action'], 'demo_id': demo['demo_id']}
                if not any(all(f[k] == fragment[k] for k in ('revision', 'stage', 'effect', 'action')) for f in self.fragments):
                    self.fragments.append(fragment)
        self.evidence = {}
        self.workflows = {}

    def state_hash(self):
        return hashlib.sha256(json.dumps(self.__dict__, sort_keys=True).encode()).hexdigest()

    def key(self, task, stage):
        return '|'.join((task.tool_revision, stage, task.request[stage]))

    def whole_key(self, task):
        return '|'.join([task.tool_revision] + [task.request[s] for s in STAGES])

    def candidate_evidence(self, task, stage, action):
        return self.evidence.get(self.key(task, stage), {}).get(action, {'success': 0, 'failure': 0,
                                                                        'accepted_tasks': [], 'reward': 0})

    def incentive_score(self, record):
        return (record['success'] + 1) / (record['success'] + record['failure'] + 2)

    def initial_program(self, task):
        program = dict(DEFAULT_PROGRAM)
        if self.method == 'pi_no_incentive': return program
        if self.method == 'trajectory_replay':
            record = self.workflows.get(self.whole_key(task))
            if record and len(record['tasks']) >= 2: return dict(record['program'])
            return program
        for stage in STAGES:
            options = self.evidence.get(self.key(task, stage), {})
            eligible = [(action, record) for action, record in options.items()
                        if len(record['accepted_tasks']) >= 2
                        and (self.incentive_score(record) >= .7 if self.method != 'compositional_repair' else True)]
            if eligible:
                program[stage] = max(eligible, key=lambda ar: (self.incentive_score(ar[1]), ar[1]['reward']))[0]
        return program

    def ranked_templates(self, task, stage=None, goal=None):
        def score(template):
            similarity = sum(template['request'][s] == task.request[s] for s in STAGES)
            compatible = template['revision'] == task.tool_revision
            goal_match = stage is not None and template['request'][stage] == goal
            return (4 * goal_match + 2 * compatible + similarity, template['demo_id'])
        return sorted(self.templates, key=score, reverse=True)

    def local_candidates(self, task, stage, goal):
        if self.method == 'self_trial_only': return list(ACTION_IDS[stage])
        if self.method == 'pi_no_goal':
            actions = [template['program'][stage] for template in self.ranked_templates(task)]
        else:
            candidates = [f for f in self.fragments if f['stage'] == stage and f['effect'] == goal]
            candidates.sort(key=lambda f: (f['revision'] == task.tool_revision,
                                           self.incentive_score(self.candidate_evidence(task, stage, f['action']))
                                           if self.method != 'compositional_repair' else 0,
                                           f['demo_id']), reverse=True)
            actions = [f['action'] for f in candidates]
        if self.method != 'pi_no_incentive':
            known = sorted(self.evidence.get(self.key(task, stage), {}).items(),
                           key=lambda ar: self.incentive_score(ar[1]), reverse=True)
            actions = [action for action, record in known if record['success'] > 0] + actions
        return list(dict.fromkeys(actions + list(ACTION_IDS[stage])))

    def reinforce(self, task, updates, final_program, final_feedback):
        if self.method == 'pi_no_incentive': return
        if self.method == 'trajectory_replay':
            if final_feedback['accepted']:
                key = self.whole_key(task)
                record = self.workflows.setdefault(key, {'program': dict(final_program), 'tasks': []})
                if task.task_id not in record['tasks']: record['tasks'].append(task.task_id)
            return
        for update in updates:
            stage, action = update['stage'], update['action']
            record = self.evidence.setdefault(self.key(task, stage), {}).setdefault(action,
                      {'success': 0, 'failure': 0, 'accepted_tasks': [], 'reward': 0})
            record['success' if update['passed'] else 'failure'] += 1
            record['reward'] += update['reward']
        if final_feedback['accepted']:
            for stage, action in final_program.items():
                record = self.evidence.setdefault(self.key(task, stage), {}).setdefault(action,
                          {'success': 0, 'failure': 0, 'accepted_tasks': [], 'reward': 0})
                record['success'] += 1
                record['reward'] += 1
                if task.task_id not in record['accepted_tasks']: record['accepted_tasks'].append(task.task_id)


def episode(controller, task, submit, budget=3, learn=False, detailed=False):
    program = controller.initial_program(task)
    first_program = dict(program)
    first = submit(task, program)
    current_feedback = first['feedback']
    initial_feedback = copy.deepcopy(current_feedback)
    full_work, probe_work, trials, events, updates = first['work'], 0, 0, [], []
    full_submissions = 1
    regression_commits = 0
    goals_before = diagnosed_goals(task, current_feedback)
    micro = micro_task(task)
    seen = set()
    if not current_feedback['accepted'] and controller.method == 'pi_no_trial':
        for diagnosis in goals_before:
            stage, goal = diagnosis['stage'], diagnosis['goal']
            program[stage] = controller.local_candidates(task, stage, goal)[0]
            events.append({'scope': [stage], 'goal': goal, 'action': program[stage],
                           'trial': False, 'reason': diagnosis['explanation']})
    elif not current_feedback['accepted']:
        if controller.method == 'trajectory_replay':
            best_failures = failed_count(current_feedback)
            for template in controller.ranked_templates(task):
                if trials >= budget: break
                proposal = dict(template['program'])
                signature = json.dumps(proposal, sort_keys=True)
                if signature in seen or proposal == program: continue
                seen.add(signature)
                result = submit(micro, proposal); trials += 1; probe_work += result['work']
                feedback = result['feedback']
                committed = failed_count(feedback) < best_failures
                regressions = sum(value and not feedback['passed'][assertion]
                                  for assertion, value in current_feedback['passed'].items())
                if committed:
                    program = proposal; current_feedback = feedback; best_failures = failed_count(feedback)
                    regression_commits += regressions
                events.append({'scope': list(STAGES), 'goal': 'complete_request',
                               'candidate_source': template['demo_id'], 'program': proposal,
                               'probe_passed': feedback['accepted'], 'committed': committed,
                               'regressions': regressions if committed else 0,
                               'failed_constraints': failed_count(feedback)})
                if feedback['accepted']: break
        else:
            while trials < budget:
                diagnoses = diagnosed_goals(task, current_feedback)
                if not diagnoses: break
                diagnosis = diagnoses[0]
                stage, goal = diagnosis['stage'], diagnosis['goal']
                proposals = []
                if controller.method == 'pi_no_localization':
                    for template in controller.ranked_templates(task, stage, goal):
                        proposals.append((dict(template['program']), template['demo_id']))
                else:
                    for action in controller.local_candidates(task, stage, goal):
                        proposal = dict(program); proposal[stage] = action
                        proposals.append((proposal, 'effect_fragment_or_own_trial'))
                proposed = next(((proposal, source) for proposal, source in proposals
                                 if proposal != program and
                                 (stage, json.dumps(proposal, sort_keys=True)) not in seen), None)
                if proposed is None: break
                proposal, source = proposed
                seen.add((stage, json.dumps(proposal, sort_keys=True)))
                result = submit(micro, proposal); trials += 1; probe_work += result['work']
                feedback = result['feedback']
                passed = feedback['passed'][diagnosis['assertion']]
                reward = failed_count(current_feedback) - failed_count(feedback) - .01 * result['work'] / max(1, len(micro.rows))
                # A localized pass is tentative. Full-task acceptance is required
                # for proactive future reuse of this conditioned action.
                updates.append({'stage': stage, 'action': proposal[stage], 'passed': passed,
                                'reward': reward if passed else min(0, reward)})
                regressions = sum(value and not feedback['passed'][assertion]
                                  for assertion, value in current_feedback['passed'].items())
                if passed:
                    program = proposal; current_feedback = feedback
                    regression_commits += regressions
                events.append({'scope': list(STAGES) if controller.method == 'pi_no_localization' else [stage],
                               'goal': goal, 'assertion': diagnosis['assertion'], 'reason': diagnosis['explanation'],
                               'candidate_source': source, 'program': proposal, 'probe_passed': passed,
                               'committed': passed, 'regressions': regressions if passed else 0,
                               'incentive': reward if passed else min(0, reward),
                               'failed_constraints': failed_count(feedback)})
    if not first['feedback']['accepted']:
        final = submit(task, program); full_submissions += 1; full_work += final['work']
    else:
        final = first
    final_feedback = final['feedback']
    if learn: controller.reinforce(task, updates, program, final_feedback)
    regression = any(value and not final_feedback['passed'][assertion]
                     for assertion, value in initial_feedback['passed'].items())
    assert full_submissions <= 2 and trials <= budget
    result = {'task_id': task.task_id, 'group': task.group, 'tool_revision': task.tool_revision,
              'success': int(final_feedback['accepted']), 'first_success': int(first['feedback']['accepted']),
              'initial_failed_constraints': failed_count(initial_feedback),
              'final_failed_constraints': failed_count(final_feedback),
              'small_trials': trials, 'full_submissions': full_submissions,
              'row_stage_visits': full_work + probe_work, 'probe_row_stage_visits': probe_work,
              'final_collateral_regression': int(regression),
              'collateral_regression_commits': regression_commits,
              'probe_failures': sum(e.get('probe_passed') is False for e in events),
              'accepted_incentive': sum(e.get('incentive', 0) for e in events if e.get('committed')),
              'initial_program': first_program, 'final_program': dict(program), 'events': events}
    if detailed:
        result.update(task=asdict(task), initial_response=first, final_response=final)
    return result


def frozen_batch(environment, controller, tasks, budget, seed, mode):
    before = controller.state_hash()
    rows = []
    for task in tasks:
        row = episode(controller, task, environment.submit, budget=budget)
        row.update(seed=seed, method=controller.method, mode=mode, trial_budget=budget)
        rows.append(row)
    assert controller.state_hash() == before, 'Persistent state changed during frozen evaluation'
    return rows


def summarize(rows, keys):
    grouped = defaultdict(list)
    for row in rows: grouped[tuple(row[k] for k in keys)].append(row)
    metrics = ['success', 'first_success', 'initial_failed_constraints', 'final_failed_constraints',
               'small_trials', 'full_submissions', 'row_stage_visits', 'probe_row_stage_visits',
               'final_collateral_regression', 'collateral_regression_commits', 'probe_failures', 'accepted_incentive']
    result = []
    for group, entries in sorted(grouped.items()):
        item = dict(zip(keys, group)); item['episodes'] = len(entries)
        for metric in metrics: item[metric] = statistics.mean(r[metric] for r in entries)
        result.append(item)
    return result


def write_csv(path, rows):
    with path.open('w', encoding='utf-8-sig', newline='') as f:
        writer = csv.DictWriter(f, list(rows[0])); writer.writeheader(); writer.writerows(rows)


def experiment():
    OUT.mkdir(parents=True, exist_ok=True)
    start = time.perf_counter()
    final, acquisition, curves, examples, corpora, task_archive, snapshots, checks = [], [], [], [], [], [], [], []
    for seed in SEEDS:
        environment = Environment(seed)
        corpus = demonstrations(environment, seed)
        corpora.extend(dict(seed=seed, **entry) for entry in corpus)
        train = acquisition_tasks(seed); tests = evaluation_tasks(seed); probes = evaluation_tasks(seed, 'probe', 24)
        task_archive.extend({'seed': seed, **asdict(task)} for task in train + tests + probes)
        for method in METHODS:
            controller = Controller(method, corpus)
            final.extend(frozen_batch(environment, controller, tests, 3, seed, 'cold'))
            for checkpoint in (0, 24, 48, 96):
                if checkpoint:
                    previous = {24: 0, 48: 24, 96: 48}[checkpoint]
                    for task in train[previous:checkpoint]:
                        row = episode(controller, task, environment.submit, budget=3, learn=True)
                        row.update(seed=seed, method=method, mode='acquisition', trial_budget=3)
                        acquisition.append(row)
                rows = frozen_batch(environment, controller, probes, 3, seed, f'checkpoint_{checkpoint}')
                for row in rows: row['checkpoint'] = checkpoint
                curves.extend(rows)
            for budget in (1, 3, 6):
                final.extend(frozen_batch(environment, controller, tests, budget, seed, 'trained'))
            checks.append({'seed': seed, 'method': method, 'frozen_evaluation': True,
                           'condition_keys': len(controller.evidence), 'whole_workflows': len(controller.workflows)})
            if method == 'pi_full':
                snapshots.append({'seed': seed, 'evidence': controller.evidence,
                                  'fragments': controller.fragments, 'state_hash': controller.state_hash()})
                reset = Controller(method, corpus)
                reset_rows = frozen_batch(environment, reset, tests, 3, seed, 'reset_verified_state')
                final.extend(reset_rows)
            if seed == SEEDS[0] and method in ('pi_full', 'trajectory_replay', 'pi_no_localization',
                                              'pi_no_trial', 'pi_no_goal'):
                fresh = Controller(method, corpus)
                for index in (60, 61, 62, 120, 180, 181, 182):
                    row = episode(fresh, tests[index], environment.submit, budget=3, detailed=True)
                    row.update(seed=seed, method=method, mode='cold_trace_example', trial_budget=3)
                    examples.append(row)
        print(f'seed {seed}: completed', flush=True)
    write_csv(OUT/'summary.csv', summarize(final, ['mode', 'trial_budget', 'method', 'group']))
    write_csv(OUT/'per_seed_summary.csv', summarize(final, ['seed', 'mode', 'trial_budget', 'method', 'group']))
    write_csv(OUT/'learning_curves.csv', summarize(curves, ['checkpoint', 'method', 'group']))
    write_csv(OUT/'acquisition_summary.csv', summarize(acquisition, ['seed', 'method']))
    for filename, records in [('episodes.jsonl', final), ('acquisition_episodes.jsonl', acquisition),
                              ('trace_examples.jsonl', examples), ('demonstrations.jsonl', corpora),
                              ('tasks.jsonl', task_archive)]:
        with (OUT/filename).open('w', encoding='utf-8') as f:
            for record in records: f.write(json.dumps(record, ensure_ascii=False) + '\n')
    (OUT/'learned_incentives.json').write_text(json.dumps(snapshots, indent=2), encoding='utf-8')
    metrics = {'status': 'executed_cpu_correction_loop', 'seeds': SEEDS, 'methods': METHODS,
               'trial_budgets': [1, 3, 6], 'full_submission_limit': 2,
               'external_demonstrations_per_seed': 24, 'acquisition_tasks_per_seed': 96,
               'final_tasks_per_seed': 240, 'main_final_episodes': len(final),
               'acquisition_episodes': len(acquisition), 'checkpoint_episodes': len(curves),
               'detailed_example_executions': len(examples),
               'total_episodes': len(final) + len(acquisition) + len(curves) + len(examples),
               'elapsed_seconds': time.perf_counter()-start,
               'llm_calls': 0, 'trained_neural_parameters': 0,
               'state_checks': checks,
               'script_sha256': hashlib.sha256(Path(__file__).read_bytes()).hexdigest(),
               'limitations': ['finite typed tools and symbolic goal parsing',
                 'structured instrumented assertion explanations available equally to all policies',
                 'schema adapter and executable composition grammar supplied',
                 'small fixtures identify this finite operator vocabulary; not arbitrary-program safety',
                 'reinforcement updates counts and preferences, not neural parameters',
                 'not LLM baseline reproductions or token-budget measurements']}
    (OUT/'metrics.json').write_text(json.dumps(metrics, indent=2), encoding='utf-8')
    print(json.dumps({k:v for k,v in metrics.items() if k not in ('state_checks','limitations')}, indent=2))


if __name__ == '__main__':
    experiment()
\end{lstlisting}
\clearpage
\subsection{Exhaustive Integrity Checks}
\begin{lstlisting}
"""Integrity checks performed before inspecting controller scores."""
import copy
import itertools
import json
from collections import Counter
from dataclasses import fields
import run_correction as p


def verify():
    task_count = program_count = 0
    for seed in p.SEEDS:
        env = p.Environment(seed)
        train, tests, probes = p.acquisition_tasks(seed), p.evaluation_tasks(seed), p.evaluation_tasks(seed, 'probe', 24)
        corpus = p.demonstrations(env, seed)
        assert len(train) == 96 and len(corpus) == 24
        assert set(Counter((t.tool_revision, tuple(t.request[s] for s in p.STAGES)) for t in train).values()) == {8}
        assert set(Counter(t.group for t in tests).values()) == {60}
        assert all(sum(t.request[s] != p.DEFAULT_GOALS[s] for s in p.STAGES) <= 1 for t in train)
        assert all(sum(d['task']['request'][s] != p.DEFAULT_GOALS[s] for s in p.STAGES) <= 1 for d in corpus)
        sets = [{t.task_id for t in train}, {t.task_id for t in tests}, {t.task_id for t in probes},
                {d['demo_id'] for d in corpus}]
        assert sum(map(len, sets)) == len(set.union(*sets))
        for demo in corpus:
            for entry in demo['trace']:
                inferred = p.infer_effect(entry)
                actual = env.registry[demo['task']['tool_revision']][entry['stage']][entry['action']]
                assert inferred == actual
        for task in train + tests + probes:
            expert = env.expert_program(task)
            golden = env.submit(task, expert)
            assert golden['feedback']['accepted']
            bad = dict(golden['artifact'], total_yuan=golden['artifact']['total_yuan'] + 1)
            assert not env.feedback(task, bad, golden['trace'])['accepted']
            bad = dict(golden['artifact'], source_rows=golden['artifact']['source_rows'][:-1])
            assert not env.feedback(task, bad, golden['trace'])['accepted']
            accepted_programs = []
            for actions in itertools.product(*(p.ACTION_IDS[s] for s in p.STAGES)):
                program = dict(zip(p.STAGES, actions))
                full = env.submit(task, program)
                micro = env.submit(p.micro_task(task), program)
                # In this deliberately finite vocabulary, the shared fixture
                # must separate actions exactly. This is NOT a claim that small
                # tests certify arbitrary programs.
                assert full['feedback']['passed'] == micro['feedback']['passed'], (seed, task.task_id, program)
                if full['feedback']['accepted']: accepted_programs.append(program)
                assert set(full['feedback']) == {'accepted', 'passed', 'diagnostics'}
                assert all('action' not in d and 'stage' not in d for d in full['feedback']['diagnostics'])
                program_count += 1
            assert accepted_programs == [expert]
            task_count += 1
        for method in p.METHODS:
            controller = p.Controller(method, corpus)
            assert len(controller.fragments) == 16
            original = controller.state_hash()
            results = p.frozen_batch(env, controller, tests[:8] + tests[-8:], 3, seed, 'integrity')
            assert controller.state_hash() == original
            assert all(r['small_trials'] <= 3 and r['full_submissions'] <= 2 for r in results)
            for result in results:
                if method not in ('pi_no_localization', 'trajectory_replay'):
                    assert all(len(e['scope']) == 1 for e in result['events'])
                if method == 'pi_no_trial': assert result['small_trials'] == 0
            for task in train[:8]:
                p.episode(controller, task, env.submit, 3, learn=True)
            if method == 'pi_no_incentive':
                assert controller.evidence == {} and controller.workflows == {}
        # Frozen evaluation does not reinforce skills; acquisition does.
        learner = p.Controller('pi_full', corpus)
        before = learner.state_hash()
        for task in train:
            p.episode(learner, task, env.submit, 3, learn=True)
        assert learner.state_hash() != before
        assert learner.evidence
    assert {f.name for f in fields(p.Task)} == {
        'task_id', 'group', 'tool_revision', 'request', 'roles', 'rows'}
    result = {'status': 'passed', 'generated_tasks_checked': task_count,
              'full_micro_program_comparisons': program_count,
              'seeds': p.SEEDS,
              'checks': ['independent evaluator and executor agreement',
                         'wrong aggregate/provenance rejected',
                         'unique identifying program for each task',
                         'micro fixture agrees with full assertions for every anonymous program',
                         'demo fragment effects inferred from actual state changes',
                         'held-out multi-goal compositions',
                         'disjoint task/example identifiers',
                         'same budgets and public feedback',
                         'feedback omits correct action and hidden fault stage',
                         'one-stage patch scope except explicit no-scope and trajectory controls',
                         'frozen evaluation hashes',
                         'no-incentive control retains no verified action evidence']}
    p.OUT.mkdir(parents=True, exist_ok=True)
    (p.OUT/'integrity_checks.json').write_text(json.dumps(result, indent=2), encoding='utf-8')
    print(json.dumps(result, indent=2))


if __name__ == '__main__':
    verify()
\end{lstlisting}
\clearpage
\subsection{Exploratory Online Adaptation}
\begin{lstlisting}
"""Exploratory online adaptation, declared after the primary frozen pilot.

Frozen unseen-revision evaluation cannot measure learning across new-revision
tasks. Here we allow 24 separate C acquisition tasks, then freeze again at
0/6/12/24 exposures. Same original seeds, K=3, operators and final C tasks.
No score-dependent change to the primary experiment.
"""
import itertools
import json
import random
import time
import run_correction as p


def adaptation_tasks(seed):
    requests = [dict(zip(p.STAGES, choices))
                for choices in itertools.product(*(p.EFFECTS[s] for s in p.STAGES))
                if sum(choice != p.DEFAULT_GOALS[s] for s, choice in zip(p.STAGES, choices)) >= 2]
    assert len(requests) == 12
    requests = requests * 2
    random.Random(seed + 5000).shuffle(requests)
    return [p.make_task(seed, f'adapt-C-{i:03}', 'online_C_acquisition', 'C', request)
            for i, request in enumerate(requests)]


def experiment():
    started = time.perf_counter()
    results, acquisition, archived_tasks = [], [], []
    methods = ['pi_full', 'pi_no_incentive', 'self_trial_only', 'compositional_repair']
    warmup_episodes = 0
    for seed in p.SEEDS:
        env = p.Environment(seed)
        corpus = p.demonstrations(env, seed)
        stream = adaptation_tasks(seed)
        tests = [task for task in p.evaluation_tasks(seed) if task.group == 'new_tool_revision']
        assert not {t.task_id for t in stream}.intersection(t.task_id for t in tests)
        archived_tasks.extend({'seed': seed, **p.asdict(task)} for task in stream)
        for method in methods:
            controller = p.Controller(method, corpus)
            for task in p.acquisition_tasks(seed):
                p.episode(controller, task, env.submit, budget=3, learn=True)
                warmup_episodes += 1
            previous = 0
            for checkpoint in (0, 6, 12, 24):
                for task in stream[previous:checkpoint]:
                    row = p.episode(controller, task, env.submit, budget=3, learn=True)
                    row.update(seed=seed, method=method, mode='online_C_acquisition', trial_budget=3)
                    acquisition.append(row)
                rows = p.frozen_batch(env, controller, tests, 3, seed, 'after_online_C')
                for row in rows: row['checkpoint'] = checkpoint
                results.extend(rows)
                previous = checkpoint
    p.write_csv(p.OUT/'online_adaptation_summary.csv', p.summarize(results, ['checkpoint', 'method', 'group']))
    p.write_csv(p.OUT/'online_adaptation_per_seed.csv', p.summarize(results, ['seed', 'checkpoint', 'method', 'group']))
    p.write_csv(p.OUT/'online_adaptation_acquisition.csv', p.summarize(acquisition, ['seed', 'method']))
    for filename, records in [('online_adaptation_episodes.jsonl', results),
                              ('online_adaptation_tasks.jsonl', archived_tasks),
                              ('online_adaptation_learning.jsonl', acquisition)]:
        with (p.OUT/filename).open('w', encoding='utf-8') as f:
            for row in records: f.write(json.dumps(row) + '\n')
    metrics = {'status': 'executed_exploratory_online_adaptation',
               'declared_after_primary_scores': True,
               'primary_protocol_unchanged': True,
               'new_C_acquisition_tasks_per_seed': 24, 'trial_budget': 3,
               'checkpoints': [0, 6, 12, 24], 'methods': methods,
               'warmup_episodes': warmup_episodes, 'adaptation_episodes': len(acquisition),
               'frozen_evaluation_episodes': len(results),
               'total_episodes': warmup_episodes + len(acquisition) + len(results),
               'elapsed_seconds': time.perf_counter() - started,
               'llm_calls': 0,
               'limitation': 'After C acquisition, this is retention in a learned C environment, not zero-shot transfer to an unseen tool revision.'}
    (p.OUT/'online_adaptation_metrics.json').write_text(json.dumps(metrics, indent=2), encoding='utf-8')
    print(json.dumps(metrics, indent=2))


if __name__ == '__main__':
    experiment()
\end{lstlisting}
\end{document}